\documentclass[]{article}
\usepackage{afterpage} 
\usepackage{proceed2e}
\usepackage{times} 
\usepackage{subscript} 
\usepackage[geometry]{ifsym} 
\usepackage{color} 
\usepackage[utf8]{inputenc} 
\usepackage{amsmath,amssymb,amsthm} 
\usepackage{graphicx} 
\usepackage{subfigure} 
\usepackage{booktabs} 
\usepackage{multirow} 
\usepackage{multicol}
\usepackage[round]{natbib}
\usepackage{url}
\usepackage{caption}
\usepackage{paralist}
\usepackage{enumitem} 
\setlist[description]{font=\normalfont\itshape}
\usepackage{xspace} 

\newtheorem{theorem}{Theorem}
 
\newtheorem{proposition}[theorem]{Proposition}

\newtheorem{definition}[theorem]{Definition}

\newcommand{\Tgrm}[0]{T\textsubscript{GRM}\xspace}
\newcommand{\Tloom}[0]{T\textsubscript{LOOM}\xspace}

\title{Generalized Regressive Motion: a Visual Cue to Collision}
\author{Krzysztof Chalupka\\
Computation and Neural Systems\\
California Insistute of Technology\\
E-mail: kjchalup@caltech.edu
\And 
Michael Dickinson\\
Biology\\
California Institute of Technology
\And
Pietro Perona\\
Electrical Engineering\\
California Institute of Technology}

\begin{document}
\maketitle

\section*{Abstract}
Brains and sensory systems evolved to guide motion. Central to this task is controlling the approach to stationary obstacles and detecting moving organisms. Looming has been proposed as the main monocular visual cue for detecting the approach of other animals and avoiding collisions with stationary obstacles. Elegant neural mechanisms for looming detection have been found in the brain of insects and vertebrates. However, looming has not been analyzed in the context of collisions between two moving animals. We propose an alternative strategy, {\em generalized regressive motion} (GRM), which is consistent with recently observed behavior in fruit flies. Geometric analysis proves that GRM is a reliable cue to collision among conspecifics, whereas agent-based modeling suggests that GRM is a better cue than looming as a means to detect approach, prevent collisions and maintain mobility.

\section*{INTRODUCTION}
Animals move to forage, approach potential mates, chase prey, escape from predators, and to maintain their position within a group. They must do so in environments that are often cluttered by rocks, plants and other moving animals. Whether the goal is making or avoiding contact, it is valuable to detect the proximity of both stationary and moving entities.

Looming, defined by Gibson as a visual pattern expanding symmetrically on the retina~\citep{Schiff1962}, is commonly believed to be a robust and reliable monocular visual cue to impending collision\footnote{We will call `collision'  an event in  which one moving agent comes into physical contact with another, or with a stationary obstacle. In a biological context, collisions can be harmful (as when a prey is caught by a predator, or a pathogen passes from one animal to another), or beneficial (as when one ant exchanges chemical information with a nest mate, or a predator succeeds in capturing a prey). }. It is often understood as a loosely-defined group of visual stimuli rather than a specific mechanism, with an underlying idea that an object approaching with a constant velocity produces expanding patterns on the observer's retina. When an animal is stationary, looming is a sufficient cue to detect approaching objects. When a moving animal is on a collision course with a stationary obstacle, time-to-collision can be estimated from looming patterns even when distance is unknown~\citep{Lee1981,Wang1992}. It is generally accepted that looming is a cue used by various animals to avoid stationary obstacles, and elegant neural mechanisms for its detection have been unveiled. Experiments have revealed looming-sensitive neural pathways in many animals. The DCMD/LGMD neurons of the locust~\citep{Rind1992, Hatsopoulos1995, Gabbiani2002} as well as the pigeon nucleus rotundus~\citep{Sun1998} and the goldfish Mauthner cell~\citep{Preuss2006} respond to divergence of image edges. Finally, the fruitfly uses looming-sensitive neurons during navigation (\citet{Fotowat2009, Vries2012}; see also Discussion). 

\begin{figure}[t!]
\centering
\includegraphics[width=.45\textwidth]{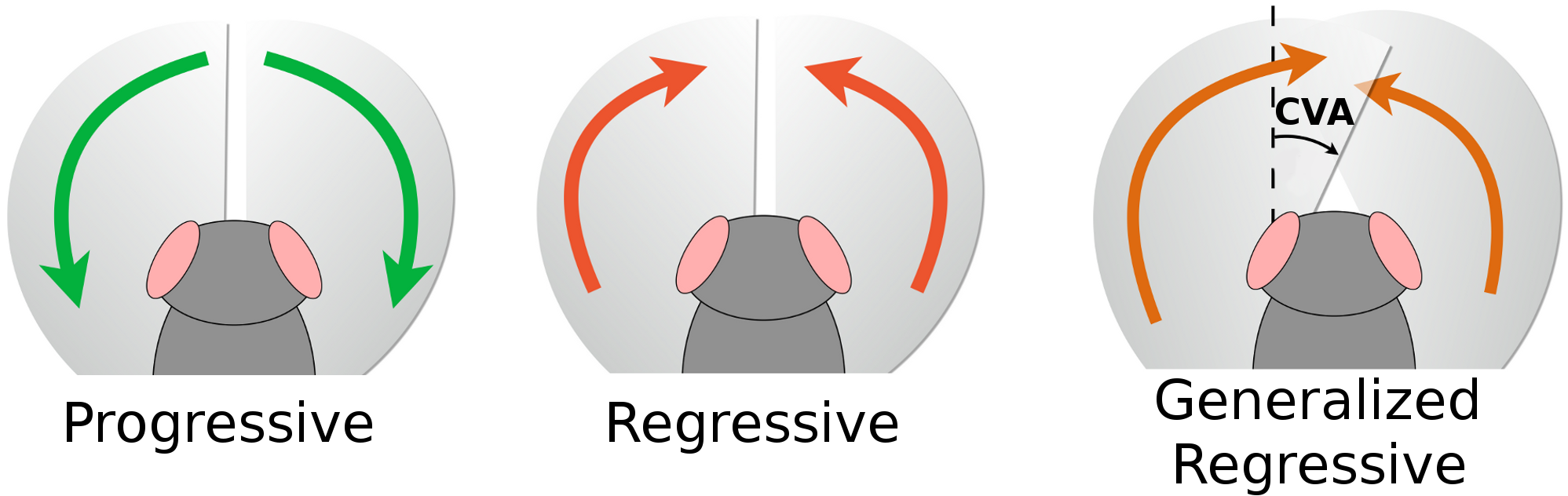}
\caption{Progressive motion, regressive motion, and GRM. Progressive motion is any counter-clockwise motion perceived in the left visual hemifield and any clockwise motion perceived in the right visual hemifield. Regressive motion is any clockwise motion in the left visual hemifield and counter-clockwise motion in the right visual hemifield. GRM is motion towards the nasal boundary of either eye. Its perception depends on the azimuthal position of the nasal boundary of the visual field of each eye. We call the angular distance of the nasal visual boundary from the 0$^\circ$ (straight-ahead) azimuth the Contralateral Visual Angle.}
\label{fig:GRM}
\end{figure}

Looming has been analyzed in the setting where either the animal or the obstacle is stationary. This is in contrast to regressive motion, used by the fruitfly to avoid collisions among multiple moving animals~\citep{Zabala2012}. Neural mechanisms behind regressive motion-driven behavior are unknown~\citep{Zabala2012}, but correlational motion detectors likely used by the fruitfly~\citep{Eichner2011, Takemura2013} can form a solid basis for regressive motion detection. The ecological usefulness of regressive motion has not been explored by~\citet{Zabala2012}. We build on the results of~\citeauthor{Zabala2012} by providing a theoretical and practical analysis of Generalized Regressive Motion (GRM) -- visual stimulus similar to regressive motion but simpler to detect and more versatile. Whereas regressive motion occurs when there is clockwise motion in the left visual hemifield or counter-clockwise motion in the right visual hemifield, GRM occurs if there is clockwise motion in the left eye or counter-clockwise motion in the right eye, as shown in Fig.~\ref{fig:GRM}. Experiments by~\citeauthor{Zabala2012} admit the hypothesis that the fruitfly uses GRM, not pure regressive motion, for stopping. We show that GRM enjoys the advantages of both looming motion and regressive motion. Our contribution is threefold:

\begin{compactitem}
\item Whereas regressive motion alone is not a good cue to frontal collisions, we use geometric reasoning to show that GRM is a sufficient cue to prevent collisions whether both agents move, or one is a stationary obstacle.
\item We argue that collisions ought not to be studied as an all-or nothing phenomenon. Rather the probability of avoiding collisions (here called `safety') is a more informative parameter. We point out that avoiding unwarranted stops is an equally important performance criterion, which we call `mobility'. 
\item With the help of agent-based modeling, we show that a population of Braitenberg-vehicle-like agents \citep{Braitenberg1986} using GRM as their sole collision-avoidance mechanism can be both safe from collisions and mobile when compared to looming-based agents.
\end{compactitem}

\section*{GEOMETRY OF REGRESSIVE MOTION}
Our geometric analysis of GRM is based on an abstract model of an agent: It is a point in the Euclidean plane, equipped with two ``eyes'' -- centers of projection. Each agent has a well-defined orientation, which allows us to define its Contralateral Visual Angle (CVA). The CVA is the angle subtended by the nasal boundary of each eye, as in Fig.~\ref{fig:GRM} (right). For now, we assume the distance between the eyes is zero and identify their position with the position of the agent (see Fig.~\ref{fig:generalizedRM}). This assumption is justified if the modeled animal's inter-eye distance is small compared to its typical distance from other animals. We nevertheless drop this assumption in simulations (described below), where we consider agents with two separate eyes and spatially extended bodies.

\begin{figure}
  \begin{center}
    \includegraphics[width=0.5\textwidth]{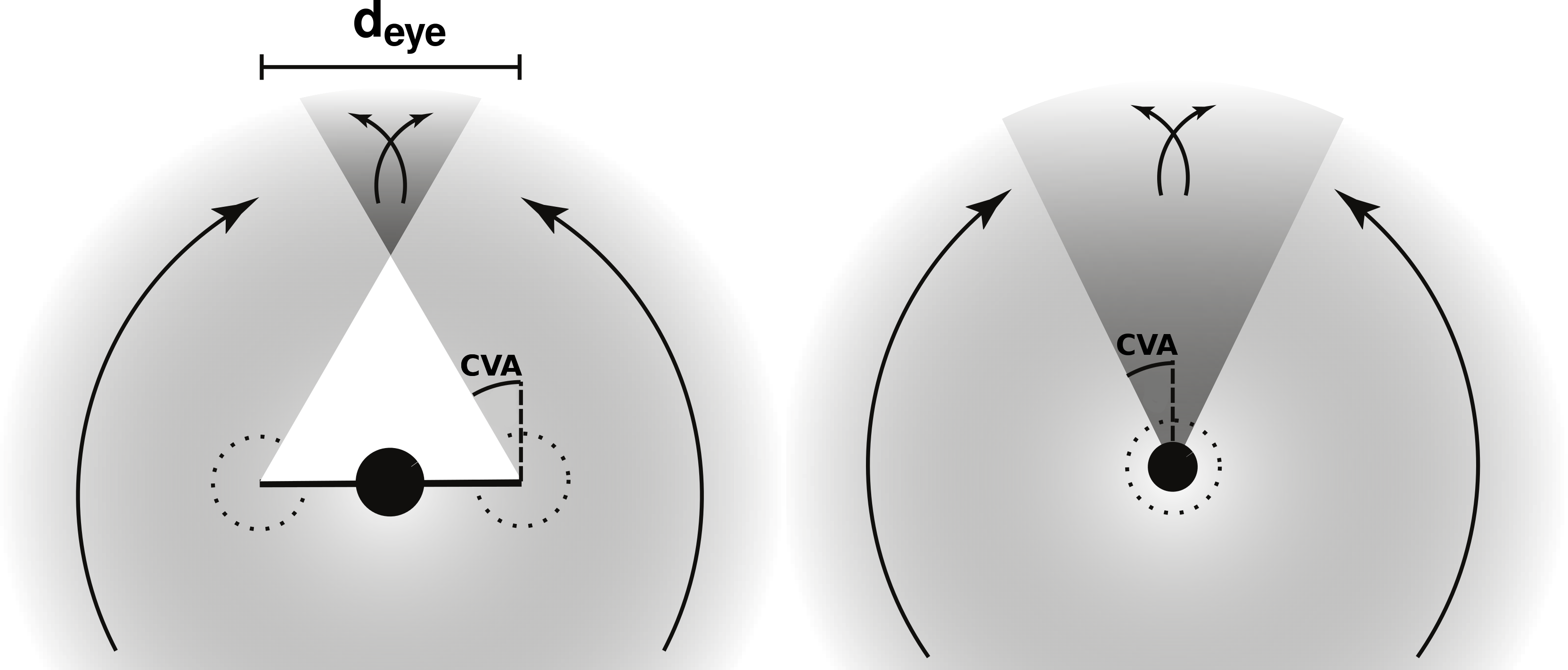}
  \end{center}
  \caption{A geometric-point agent using GRM. \textbf{Left} The arrows indicate directions of angular velocity on the agent's eyes that cause it to perceive GRM. $d_{eye}$ indicates the agent's inter-eye distance. \textbf{Right} If $d_{eye}=0$ the two eyes coincide, but still detect GRM independently. As a result, any movement in a cone extending from the agent's eyes forward is GRM.}
  \label{fig:generalizedRM}
\end{figure}

\subsection*{GRM detection can prevent all collisions among moving conspecifics}
Let a point (which could be stationary or in
movement) project to azimuthal position $\phi$ on the observer agent's eye. Let $\phi=0$ be the direction in front of the agent, positive angles for the left side and negative angles for the right side, and restrict $\pi\in [-\pi,\pi)$. Denote the point's angular velocity as $\dot{\phi}$.\\

\begin{definition} 
  A point projecting at $\phi$ is in
    \emph{regressive motion} with respect to the observer if $\dot{\phi} \cdot \phi \leq 0$, and in
    progressive motion otherwise.
\end{definition}

Before proving our main theorem, we state an easy propsition whose proof we relegate to the Mathematical Appendix.

\begin{proposition} 
  \label{prop:regMotionBasic}
  Let the relative position and velocity of the
  observed object be $\mathbf{x} \text{ and } \mathbf{v}$
  respectively. Then $\dot{\phi} = \frac{1}{\|\mathbf{x}\|^2}\langle \mathbf{v}^\perp,
  \mathbf{x} \rangle$. In particular, $\dot{\phi}$ scales as one over
  distance squared.
\end{proposition}

\begin{figure}
  \begin{center}
    \includegraphics[width=0.5\textwidth]{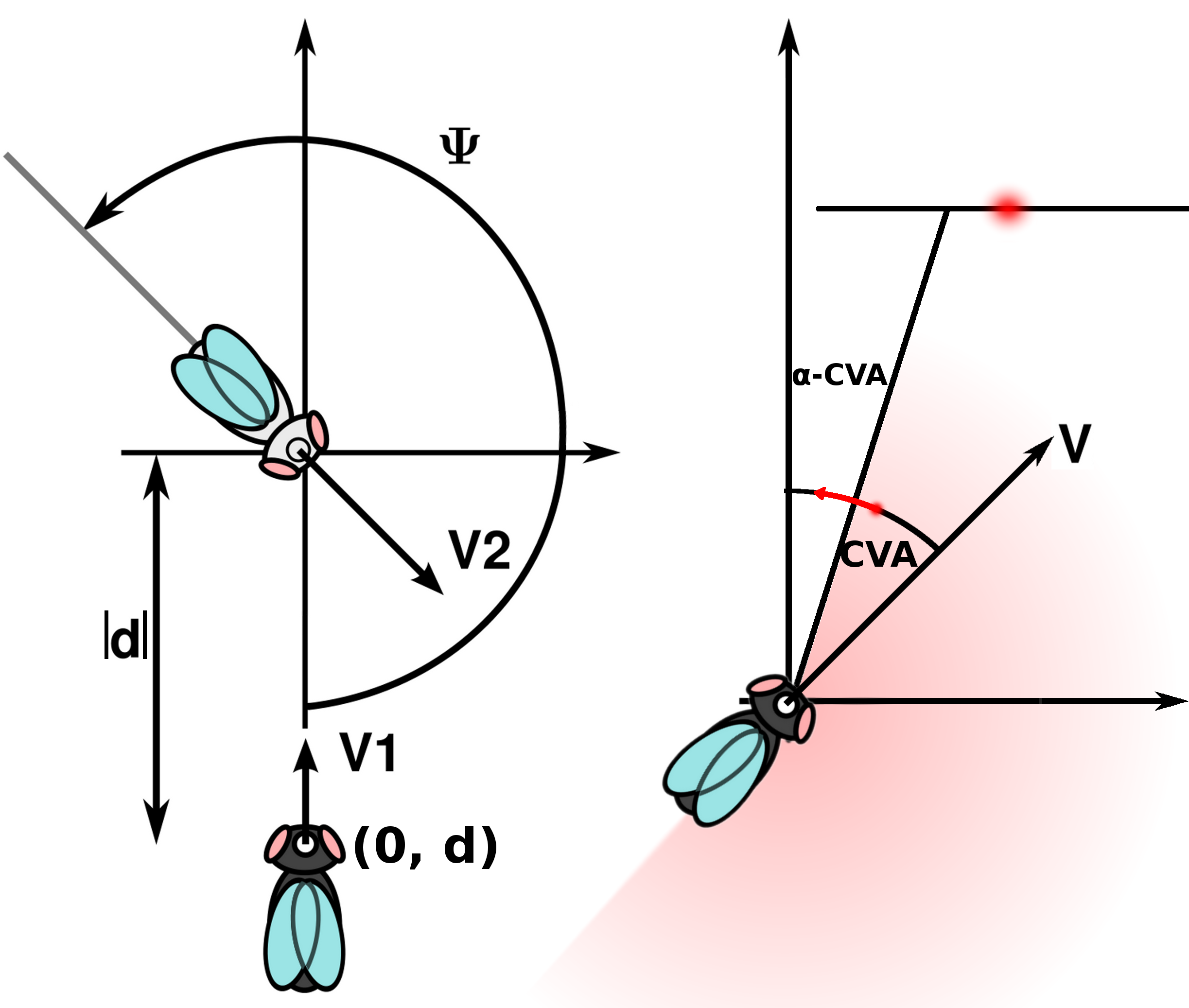}
  \end{center}
  \caption{Reference frames for Theorems~\ref{thm:GRM1}~and~\ref{thm:GRM2}. \textbf{Left} When two agents---schematically shown as flies---have non-parallel velocities, we can describe every possible pair of trajectories using four parameters (see text). \textbf{Right} When an agent with positive CVA approaches a flat wall sufficiently close, it can always observe GRM for some points on the wall (see text).}
  \label{fig:parametrized_approach}
\end{figure}

The following theorem ensures that before any potential collision, one of the agents will  perceive regressive motion.

\begin{theorem} Let f1 and f2 be two agents moving on straight, intersecting trajectories. If f1 reaches the trajectory intersection after f2, f1 perceives regressive motion at all times before f2 reaches the intersection and progressive motion afterwards. f2 perceives progressive motion before f1 reaches the intersection, and regressive motion afterward.
\label{thm:GRM1}
\end{theorem}
\begin{proof}
Let two point-agents f1, f2 move on a flat uniform surface with constant
velocities on intersecting (that is, non-parallel) trajectories. Align the reference frame's y-axis with the
direction of f1's movement, and place the origin at the point at which
the agents' trajectories cross. The situation is fully described by four
parameters (see Fig. \ref{fig:parametrized_approach} Left):

\begin{description}
\item[$v^1,v^2$] -- the speed of f1 and f2 respectively,
\item[$\Psi$] -- the angle f2's velocity vector makes with f1's
  velocity (also called the \textit{angle of approach}), and
\item[$d$] -- the y-coordinate of f1 at the moment when f2 reaches the origin.
\end{description} 

  First, we compute the angular position $\phi_{21}$ of f2 on
  f1's projection center and the angular velocity $\dot{\phi}_{21}$,
  at the moment when f1 is at distance $d+\epsilon$ from $0$.

  For $\epsilon=0$, the positions of the two agents are
  respectively
  \begin{align*}
    \mathbf{x}^1(d) &= (0, d),\\
    \mathbf{x}^2(d) &= (0, 0).
  \end{align*}
  If $\epsilon \neq 0$, the time that passed since the original
  configuration is $\Delta t = \epsilon/
  v^1$, and since the velocity of
  f2 is $\mathbf{v}^2 = v^2(-\sin{\Psi}, \cos{\Psi})$, we have 
  \begin{align*}
    \mathbf{x}^1(d+\epsilon) &= (0, d+\epsilon),\\
    \mathbf{x}^2(d+\epsilon) &= \epsilon\frac{v^2}{v^1}(-\sin{\Psi}, \cos{\Psi}).
  \end{align*}
  Define $\mathbf{x}^R, \mathbf{v}^R$ to be the relative position and velocity of f2 in
  f1's frame of reference. Then
  \begin{align*}
    \mathbf{x}^R(d+\epsilon) &= \mathbf{x}^2(d+\epsilon) - \mathbf{x}^1(d+\epsilon)\\
    &= (-\frac{\epsilon v^2}{v^1}\sin{\Psi},
    \frac{\epsilon v^2}{v^1}\cos{\Psi}-(d+\epsilon)),\\
    \mathbf{v}^R(d+\epsilon) &= (-v^2\sin{\Psi}, v^2\cos{\Psi}-v^1).
  \end{align*}
  From this we can directly compute $\phi_{21}$, and Proposition 2
  enables us to compute the angular velocity of f2:
  \begin{align*}
    \phi_{21} &= \arctan{(x^R_2/x^R_1)}-\pi/2,\\
    \dot{\phi}_{21} &= \frac{1}{D^2}\langle \mathbf{v}^{R\perp}, \mathbf{x}^R\rangle,
  \end{align*}
  where $D^2=(x^R_1)^2+(x^R_2)^2$ is the distance between the two
  agents. Plugging in the values calculated above we get
  \begin{align}
    \phi_{21}(d,\epsilon) &=
    \arctan{\frac{\epsilon\frac{v^2}{v^1}\cos{\Psi}-(d+\epsilon)}{-\epsilon\frac{v^2}{v^1}\sin{\Psi}}}-\pi/2,\label{eq:theoryphi}\\
    \dot{\phi}_{21}(d,\epsilon) &= -\frac{1}{D^2}dv^2\sin{\Psi},\label{eq:theorydotphi}
  \end{align}
  with $D^2 =
  (\frac{v^2\epsilon}{v^1})^2+(d+\epsilon)^2-2\epsilon\frac{v^2(d+\epsilon)}{v^1}\cos{\Psi}.$
  Now, assume that f2 arrives at the intersection first, that is
  $d<0$. From Equation \ref{eq:theorydotphi} it follows that
  \[ \dot{\phi}_{21}(d,\epsilon) \geq 0 \iff \sin{\Psi} \geq 0. \]
  But this implies
  that the denominator in the arctangent in Equation \ref{eq:theoryphi}
  is nonnegative if and only if $\epsilon\leq0$,
  \[-\epsilon \frac{v_2}{v_1} \sin{\Psi} > 0 \iff \epsilon \leq 0,\]
  A positive denominator restrict the range of the arctangent to $[-\pi/2,
    \pi/2]$, and thus 
  \[\epsilon \leq 0 \iff \phi \in [-\pi, 0].\]

  Thus, 
  \[\dot{\phi}_{21}(d,\epsilon) \geq 0 \implies (\phi(d,\epsilon) \in [-\pi, 0] \iff \epsilon \leq 0).\]
  Analogously,
  \[\dot{\phi}_{21}(d,\epsilon) \leq 0 \implies (\phi(d,\epsilon) \in
        [0, \pi] \iff \epsilon \leq 0).\]
        This proves the first part of the theorem.
        The second part is proven exactly in the same way, but switching the reference frame to that of f2.
\end{proof}

The theorem easily generalizes to GRM:
\begin{definition} 
  Let f2 project onto a projection center f1, with
  azimuthal position $\phi_{21}$. f1 perceives GRM if and only if $\left(\dot{\phi}_{21}>0 \text{ and } \phi_{21} \in [-\pi, \text{CVA}]\right)$ or $\left(\dot{\phi}_{21}<0 \text{ and } \phi_{21} \in [-\text{CVA},
    \pi]\right)$, where CVA is a fixed angle.
\end{definition}

\begin{theorem} Let f1 and f2 be two agents moving on straight, intersecting trajectories. If f1 reaches the trajectory intersection after f2, f1 perceives GRM at all times before f2 reaches the intersection.\label{thm:GRM2}\end{theorem}

This follows from Theorem~\ref{thm:GRM1} because regressive motion implies GRM (see definitions above).

\subsection*{GRM detection can prevent collisions with stationary objects}
An agent using solely non-generalized regressive motion detection will
always collide with stationary obstacles, which project expanding (progressive)
patterns. However, GRM with CVA $ > 0$ and appropriate motion thresholding provides a
mechanism for stationary collision avoidance. Let \Tgrm denote the smallest magnitude of GRM that causes the agent to stop. Intuitively, larger
CVA's and smaller \Tgrm provide better obstacle detection. 

\begin{theorem} An agent on a collision course with a stationary object will perceive GRM before the collision, as long as \Tgrm $< \infty$ and CVA $>$ 0.\label{thm:GRM3}
\end{theorem}
\begin{proof}
  We assume the agent is approaching an object such that the centerline of the agent $\phi=0$ does not point directly at a (non-smooth) corner of the object. 
  We can then assume there is a neighborhood around the $\phi=0$ aziumuth that can be approximated as a wall segment. As in Fig.~\ref{fig:parametrized_approach} (right), place the origin at the position of the agent as
  shown. The agent approaches the wall at an angle $0 < \alpha < \pi/2$ and with positive speed
  $v$, its velocity $\mathbf{v}=(v\sin{\alpha}, v\cos{\alpha})$. Consider an
  arbitrary wall-point $(x,y)$, marked red in the figure. The relative
  velocity of the point w. r. t. the agent is $-\mathbf{v}$, and its angular
  velocity on the agent's eye equals (by Proposition 2)
  \begin{equation}
    \dot{\phi}(x,y,\alpha,v)=\frac{-v}{x^2+y^2}(x\cos{\alpha}-y\sin{\alpha}).\label{eq:wall_angvel}
  \end{equation}
  We need to take into account very small $\theta$ and very large $T$. the
  following proposition implies the theorem:
  \begin{equation*}
    \forall_{\theta,T>0}\:\exists_{X>0} \text{ s.t. }
    \left\{\begin{array}{l}
      \|\dot{\phi}(X,y,\alpha,v)\|>T \text{  and  }\\
      y \tan{(\alpha-\theta)}>X>y\tan{\alpha}
    \end{array},\right.
  \end{equation*}
  where the last condition restricts the point we're searching for to be
  between the leftmost edge of perceived GRM and the center of the
  agent's visual field.
  Now, fix
  $X=y\tan{(\frac{\alpha-\theta}{k})}, k>1$. $k$ can always be chosen to make the point $(X,y)$ arbitrarily close to the centerline, and so contained in the smooth wall-like neighborhood on the obstacle. Clearly these $X,y$ satisfy the
  second condition above. We also have
  \[\dot{\phi}(X,y,\alpha,v)=\frac{-v\left(y\left[\tan{\frac{\alpha-\theta}{k}}\cos{\alpha}-\sin{\alpha}\right]\right)}{y^2(\tan{\frac{\alpha-\theta}{k}}+1)}.\]
  Since $\alpha,\; \theta$ and $k$ are constant, this expression scales as $1/y$ and
  thus reaches arbitrarily large magnitudes as $y$ approaches $0$ -- that is, as we place the observer closer and closer to the wall.
\end{proof}

The intuition behind this theorem is that a GRM detector detects any motion in a cone symmetrical about the center of the visual field. If the agent frontally approaches an object, the target is guaranteed to produce a strong enough signal within that cone at some positive time before the collision occurs.

\subsection*{False alarms, safety and mobility}
Theorems~\ref{thm:GRM2}~and~\ref{thm:GRM3} give basis to the claim that GRM can be useful for collision avoidance. However, it can be argued that they are of limited practical use. One problematic area not explored by the theorems is that of false alarms. An agent using GRM as a stopping cue can stop unnecessarily in a variety of situations, some of which are shown in Fig.~\ref{fig:falseAlarms}. 

\begin{figure}
\centering
\includegraphics[width=.4\textwidth]{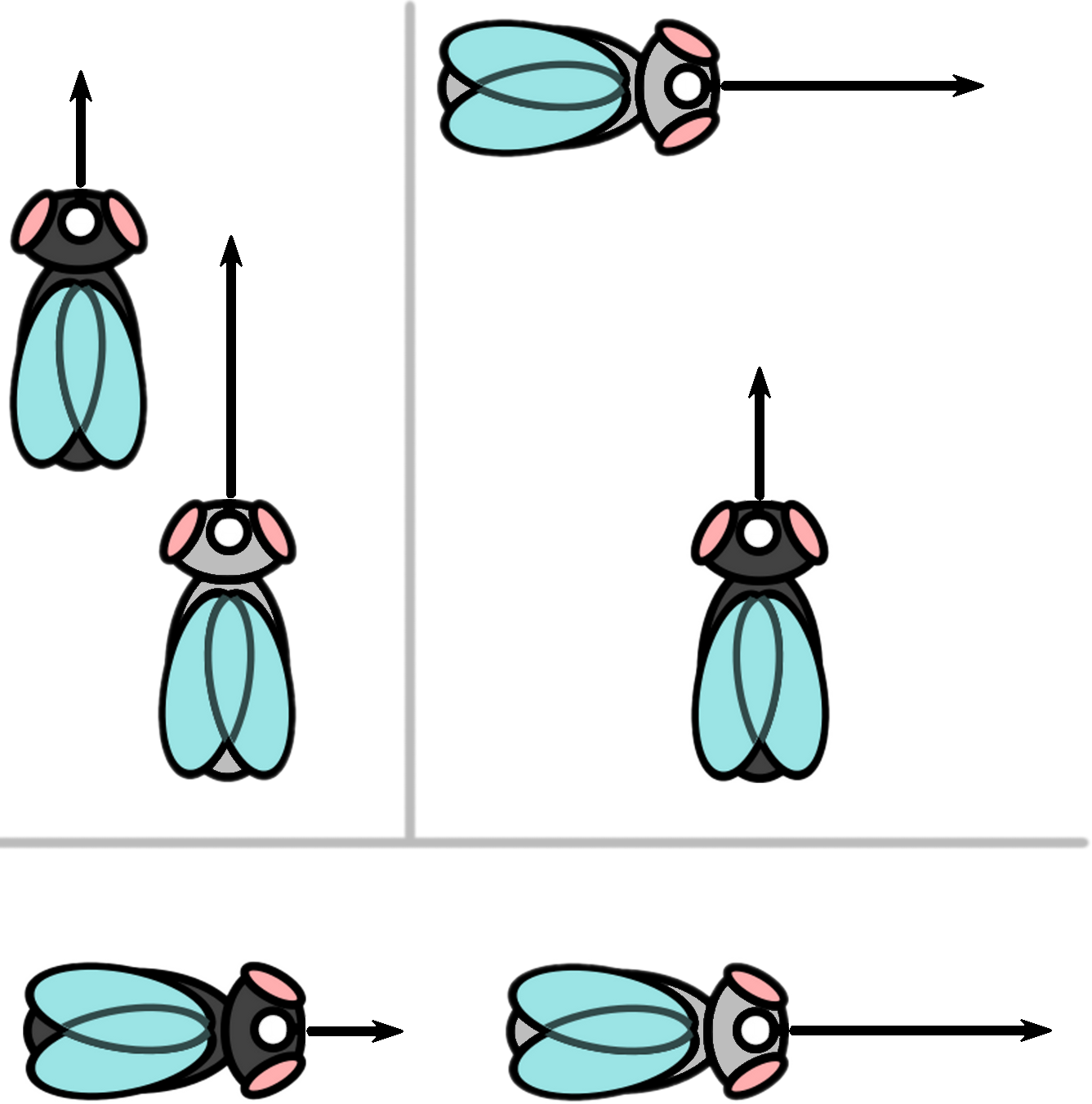}
\caption{GRM False Alarms. Arrow lengths indicate speeds. In each of the situations pictured collision is not imminent, but the darker agent perceives GRM. \textbf{Left} The faster, bright agent takes over a slower one from behind. \textbf{Right} The brighter agent crosses a trajectory junction in advance of the darker agent's arrival. \textbf{Bottom} The faster, dark agent is moving away from the slower bright agent.}
\label{fig:falseAlarms}
\end{figure}

If perception of GRM of any magnitude greater than zero caused the agents to stop, they would be perfectly safe from collisions, but unable to move. In simulations described below we investigate the role of two parameters that enable GRM-based agents to trade-off mobility and safety from collisions: the CVA and tne threshold \Tgrm on the magnitude of GRM that stops the agent. Varying these parameters in a population of agents changes the population's safety and mobility, where safety corresponds to the fraction of prevented collisions, and mobility to the fraction of useful stops. Formally, we can classify any encounter between agents f1 and f2 as
 \begin{description}
 \item[True Positive (TP)]: agent f1 stops due to perceived motion of
   agent f2, and \textit{f1 would collide with f2} had both f1 and f2 continued to move
     with velocities they had at the moment of f1's stop, and f2 is
     outside of f1's collision radius.
 \item[True Negative (TN)]: agent f1 moves without stopping and doesn't collide with any entity.
 \item[False Positive (FP)]: agent f1 stops due to perceived motion of
   entity f2, and \textit{f1 would not collide with f2} had both f1 and f2 continued to move at velocities they had at the moment of f1's stop. In addition, f2 is outside of f1's collision radius at the time of stopping.
 \item[False Negative (FN)]: agent f1 collides with agent f2.
 \end{description}

 We can then define mobility and safety as
 \begin{align*}
 \text{mobility} &= \frac{TP}{TP+FP} \text{ and}\\
 \text{safety}  &= \frac{TP}{TP+FN},
 \end{align*}
 Mobility is high if and only if the agent rarely stops without a good reason. Safety is a complementary measure that is high if and  only if the agent avoided many out of all the potential collisions. In our view, any collision avoidance algorithm is useful inasmuch as it offers a range of good mobility-safety tradeoffs: It can be used to make mobile vehicles remain relatively safe (relative to other algorithms), as well as safe vehicles that retain good mobility.

\section*{SIMULATIONS}
A number of issues are not covered by our theory:
\begin{itemize}
\item real agents have extended bodies, unlike the geometric points considered in Thm~\ref{thm:GRM2}, 
\item real agents have visual systems with multiple centers of projection and 
\item there is a tradeoff of safety and mobility to be explored.
\end{itemize}
Studying the tradeoff between mobility and safety is best done experimentally in a simulated environment. It is difficult to derive theoretical tradeoff curves given the statistical variability of the trajectories even in simple environments. Computational simulations are also a good tool for studying GRM agents with extended bodies and two separate eyes. Thus, to further study GRM we use agent-based modeling with populations of fly-like agents trying to avoid collisions. The agents use GRM- and looming-based algorithms for stopping. We compare the performance of GRM-based and looming-based collision avoidance using population safety and mobility as performance metrics. The details of the simulation setup are available in Methods below. Matlab code implementing the simulations is available online at \texttt{http://vision.caltech.edu/$\sim$kchalupk/cod} \texttt{e.html}.

\begin{figure}[h!t]
\centering
\includegraphics[width=.4\textwidth]{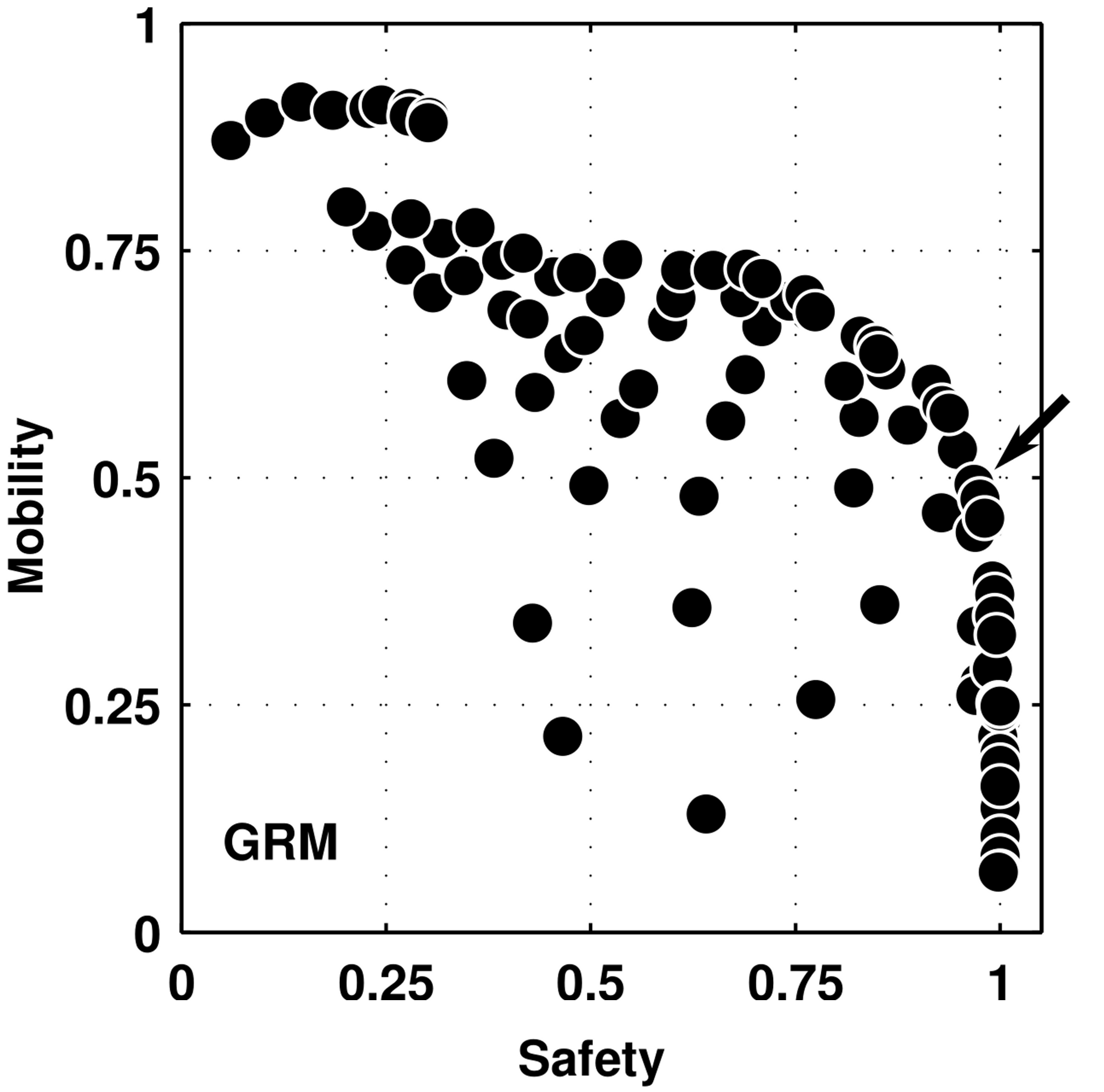}
\caption{Mobility and safety of GRM-based collision avoidance. Each point corresponds to the mean mobility and safety achieved by GRM-based agents with fixed \Tgrm and CVA (and \Tloom set to a very large value, disabling looming-based collision avoidance in practice; see Methods for more details). Each point corresponds to a different (\Tgrm, CVA) value, used in 50 randomized repetitions of 50s-long simulations to estimate the means. A wide variety of safety-mobility tradeoffs are available, including a reasonable 50\% mobility at 95\% safety marked by the arrow.}
\label{fig:results_grm}
\end{figure}

\subsection*{GRM detection offers good mobility and safety to a population of conspecifics}
In each simulation, ten fly-like agents were placed in a toroidal arena and followed straight trajectories with constant speed. Each agent was equipped with a stopping mechanism triggered by the perception of GRM with specific CVA and \Tgrm (consistent across all the agents in a given simulation run). We performed 100 types of simulations, varying the CVA and \Tgrm values systematically\footnote{The agents were also equipped with a looming detector, which in the simulations described in this section was set to be extremely insensitve.}. After running multiple trials for each simulation type, we calculated the safety and mobility of the agent population in each case. Fig.~\ref{fig:results_grm} shows that varying the GRM parameters offers a wide variety of safety-mobility tradeoffs to the population.

\subsection*{Looming detection offers poor mobility and safety to a population of conspecifics}
As a point of reference we measured the usefulness of looming as a signal for collision avoidance. To this goal, we equipped each agent with both a GRM and a looming detector. We then performed a series of simulations varying CVA, \Tgrm  and \Tloom, where the latter is a looming threshold (Methods contains a detailed description of the stopping mechanism). Fig.~\ref{fig:results_triplet} (left) shows safety and mobility in the simulations where only the looming signal was used for stopping (that is, \Tgrm was very high). The figure shows that to achieve 95\% safety, the agents had to stop unnecessarily 75\% of the time. Fig.~\ref{fig:results_triplet} (middle, right) shows full simulation data: each marker corresponds to one (CVA, \Tgrm, \Tloom) setting, and the three parameters vary independently. In Fig.~\ref{fig:results_triplet} (middle) the value of \Tgrm varies smoothly on the upper envelope of the scatter plot. It is very clear that choosing low \Tgrm offers good safety but bad mobility, whereas higher \Tgrm increases safety but decreases mobility. Fig.~\ref{fig:results_triplet} (right) shows that the correspondence between \Tloom and the safety-mobility tradeoff is much less clear. Whereas low values of \Tloom decrease mobility of the agents, increasing the value past a certain point offers no additional safety-mobility advantages when GRM is also used for collision avoidance.

\begin{figure*}
\centering
\includegraphics[width=1.\textwidth]{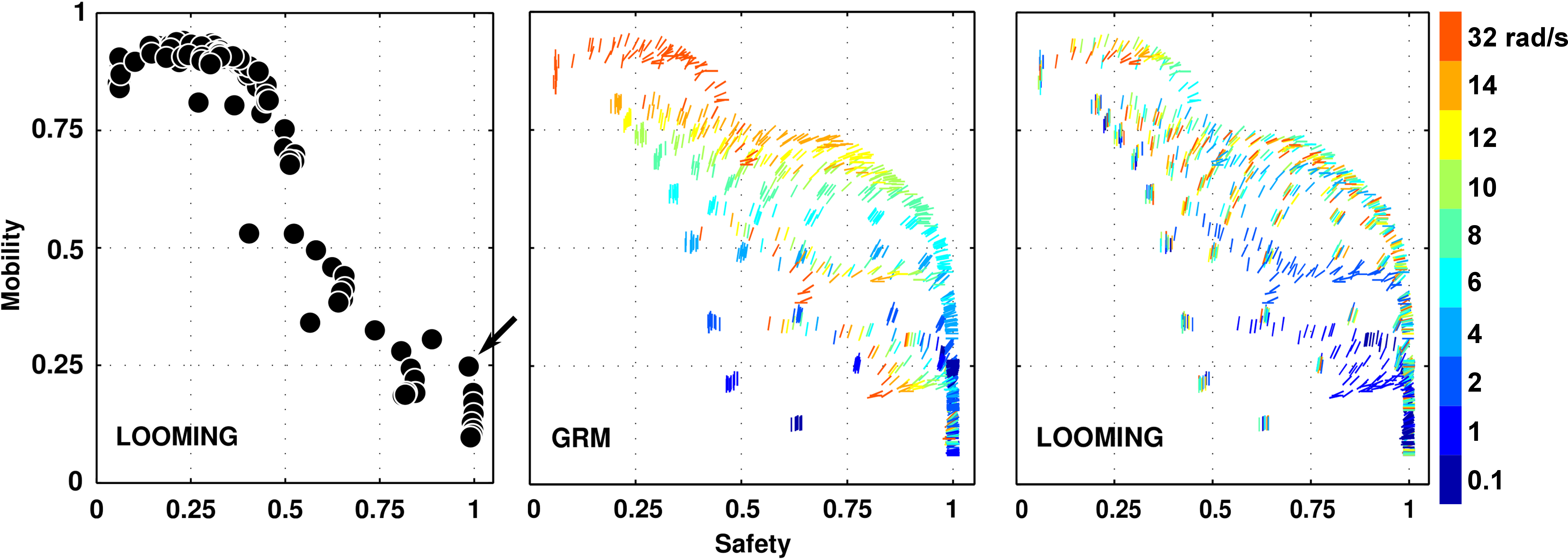}
\caption{Mobility and safety of GRM- and looming-based collision avoidance. Each point corresponds to a different (\Tgrm, \Tloom, CVA) value, used in 50 randomized repetitions of 50s-long simulations to estimate the means. Its position is at the mean safety and mobility of a population of agents over the 50 trials. \textbf{Left} Only points with \Tgrm = 32rad/s (that is, GRM detection practically disabled) are shown. The arrow points to the highest-mobility point with more than 95\% safety. \textbf{Middle} Points for all the possible (\Tgrm, \Tloom, CVA) settings. The angle of each bar is the CVA used in corresponding simulations, while the bar's color equals \Tgrm. \textbf{Right} As Middle, but the colors now correspond to \Tloom. For example, the leftmost top point corresponds to CVA=0, \Tgrm=32, \Tloom=4. This point has high mobility and low safety. This is because its low CVA and high \Tgrm can't prevent many collisions; at the same time, relatively low \Tloom doesn't seem to help much with collision avoidance.}
\label{fig:results_triplet}
\end{figure*}

\section*{METHODS: SIMULATION PARAMETERS}
 \label{sec:fly_navigation}
Each simulation puts ten fly-like agents, each defined by 14 visible points (see Fig. \ref{fig:fly_outline}), in a toroidal arena --- a square of side length 50mm with opposing edges glued together. 

\afterpage{
\begin{figure}
\centering
\includegraphics[width=.4\textwidth]{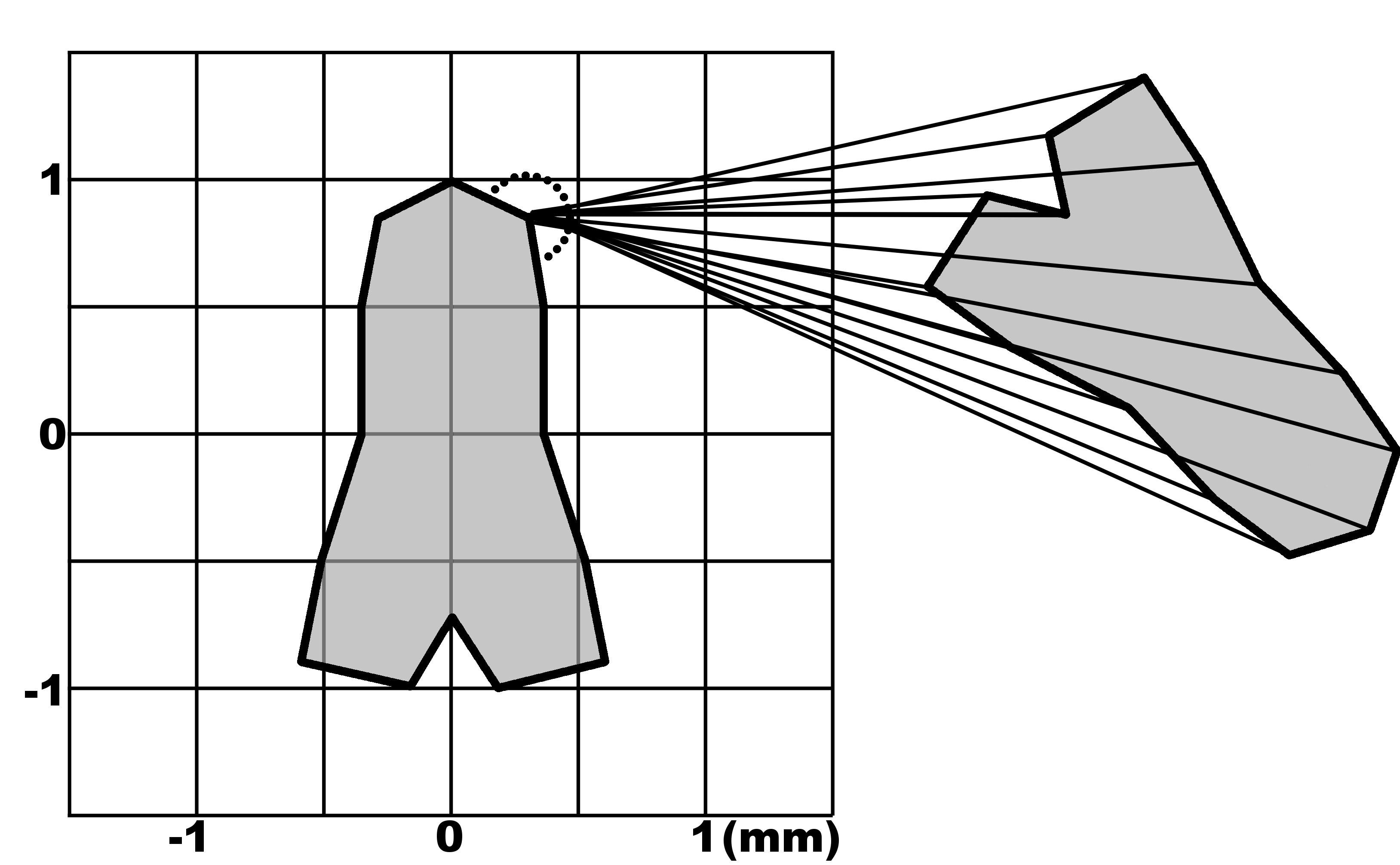}
\caption{Outline of a simulated fly agent. Each agent in our simulations is defined by 14 points, as shown in the figure. The agent's two eyes (only the right eye is shown) compute at each time step the angular projections and velocities of the 14 defining points of each other visible agent. }
\label{fig:fly_outline}
\end{figure}
}

The flies follow simple dynamics, described below. For each combination of $\text{CVA} \in \{0,10,...,90\}$ degrees and \Tgrm, \Tloom $\in \{0.1, 1, 2, 4, 6, 8, 10, 12, 14, 32\}$ rad/s, we ran 50 trials of 10000 time-steps (to a total of $1000\times 50$ simulations, 50s long each, of flies walking at realistic speeds). Movies S1-S4, available online at \texttt{http://vision.caltech.edu/$\sim$kchalupk/cod} \texttt{e.html}, show four example simulation runs resulting from different parameter settings (Movie Captions Appendix describes symbols used in the movies.)

 \begin{table*}[!ht]
   \centering
   \caption{{\bf Parameters defining the fly model.}}
   \label{tab:parameters}
   \begin{tabular}{ccl}
     \toprule
     \textbf{Symbol} & \textbf{Default value} & \textbf{Meaning}\\  
     \midrule
     $\Delta t$ & 0.005 s & Time-step in numerical simulations\\ 
     $R$ & 50 mm & Edge length of the walking arena (glued into a torus)\\ 
     $N$ & 10 & Number of flies in the arena\\ 
     $l$ & 2 mm & Length of an agent\\ 
     $d$ & 0.55 mm & distance between the eyes' centers\\ 
     $v_{\min}$ &10 mm/s & minimum agent speed \\ 
     $v_{\max}$ & 30 mm/s & maximum agent speed\\ 
     $v_i$ & $ \sim U(v_{\min}, v_{\max})$ & walking velocity of agent $i$\\ 
     $p_{01}$ & $\sim $0.8 & prob. stop-to-walk in a 1s time-interval\\ 
     $P_{01}$ & 0.008 & prob. stop-to-walk in one time-interval (i.e. $(1-p_{01})^{\Delta t} = (1-P_{01})$)\\ 
     \Tloom& 0$^\circ/s$ ($\dag$)& stopping threshold on looming motion\\
     \Tgrm& 0$^\circ/s$ ($\dag$)& stopping threshold on regressive motion\\ 
     $\text{CVA}$ & 30$^\circ$ ($\dag$) & the CVA\\ 
     $\theta_i$ & 120$^\circ$ & angle of ipsilateral visual field seen by each eye\\ 
     $\delta_\sigma$ & 30$^\circ$ &increment of standard deviation of agent reorientation motions $\sigma_{\vec{v}}$\\ 
     $\lambda_\sigma$ & 0.992 & decay constant for $\sigma_{\vec{v}}$ at each $\Delta t$ \\ 
     $n$ & 14 & number of points on each agent \\  
     &($\dag$) & variable whose value is systematically explored in some experiments\\
     \bottomrule
   \end{tabular}
 \end{table*}

 \begin{description}
 \item[Numerical implementation:] the motion and control of each agent
   are computed at discrete time-intervals with constant time
   increments of $\Delta t$. For simplicity of notation where we write
   $t+1$ in the following, we mean $t + \Delta t$. The value of $\Delta t$  is given
   in Table~\ref{tab:parameters} alongside all other simulation parameters. 
  
 \item[Agent trajectories:] the $i$-th agent's trajectory is determined by
   walking speed $v_i$, initial position $\vec{x}_i^0$ and initial
   direction $\vec{v}_i$ with $\| \vec{v}_i\| = 1$, i.e.  $\vec{x}_i(t)
   = \vec{x}_i^0 + t v_i \vec{v}_i$. Each agent has a different constant velocity chosen uniformly from the [1-3] cm/s interval.

 \item[Control:] while it is walking, each agent keeps walking at constant
   velocity until it is stopped by a GRM or looming percept caused by another agent (as explained below).

 \item[Spontaneous start:] when an agent is stationary it flips a coin at
   each time step: if the coin turns out to be heads, the agent
   starts moving in the current direction at its preferred
   velocity. If the coin is tails it stays put. The probability of
   obtaining heads in the interval of one second is denoted $p_{01}$;
   thus, the probability of obtaining tails in one time-step is $(1 -
   p_{01})^{\Delta t}$.

 \item[Orientation:] while an agent is walking, it keeps constant velocity and orientation
   (i.e. $\vec{v}_i (t+1) = \vec{v}_i(t)$). Upon stopping it samples a new orientation from a
   Gaussian pdf with standard deviation equal to the current value of a parameter $\sigma_i$ and centered at the current orientation. In addition, after reorientation, $\sigma_i$ increases by a fixed amount
   $\delta_\sigma$.  $\sigma_i$  decays exponentially, with decay constant $\lambda_\sigma$, i.e. $\sigma_i(t+k\Delta t) =
   \sigma_i(t)\lambda_\sigma^k$. This mechanism, modeling basic
   neuronal sensitization, is a simple way to allow the flies to keep
   roughly straight trajectories when encountering transient obstacles
   (other moving flies), and avoid getting stuck around large static
   obstacles (groups of stopped flies). 
  
 \item[GRM:] each eye sees GRM whenever the angular motion of any point projecting onto its retina is directed contralaterally, i.e. counterclockwise for the right eye and
   clockwise for the left eye. Each eye's visual field goes beyond the
   frontal direction to cover a given CVA. The CVA is a
   free parameter which we study to discover the best compromise between
   avoiding collisions and false alarms.

 \item[GRM stops:] agent boundaries are defined by 14 points visible to
   other flies, as shown in Fig. \ref{fig:fly_outline}. Call the $j$-th point on the $i$-th agent $p_j^i$ and its azimuthal position on the observer's eye $\phi(p_j^i)$. Each agent
   measures the angular velocity of all the points on the other flies. If {\em any} GRM $\dot{\phi}(p_j^i)$ is detected, the agent compares the GRM magnitude $\|\dot{\phi}(p_j^i)\|$ to a threshold \Tgrm and stops if $\|\dot{\phi}(p_j^i)\| >$ \Tgrm.

 \item[Looming motion:] Let $\|\dot{\phi}_L\|$ denote the  largest magnitude of counter-clockwise motion that any point evokes in the left visual hemifield, and  $\|\dot{\phi}_R\|$ the largest magnitude of clockwise motion evoked in the right visual hemifield. Then the strength of looming perceived by the agent equals $\omega_{LOOM} := \min{(\|\dot{\phi}_L\|,\|\dot{\phi}_R\|)}$. 

 \item[Looming stops:] similarly to GRM stopping, the agent stops if $\omega_{LOOM} >\;\; $ \Tloom. This simple mechanism activates only if the agent can perceive points diverging at velocities larger than \Tloom.
 \end{description}

 All simulation parameter values are specified in Table~\ref{tab:parameters}. To summarize, agent $i$'s motion is governed by the following equations
 (see also the diagram in Fig. \ref{fig:fly_control}).

\begin{figure}
\centering
\includegraphics[width=.4\textwidth]{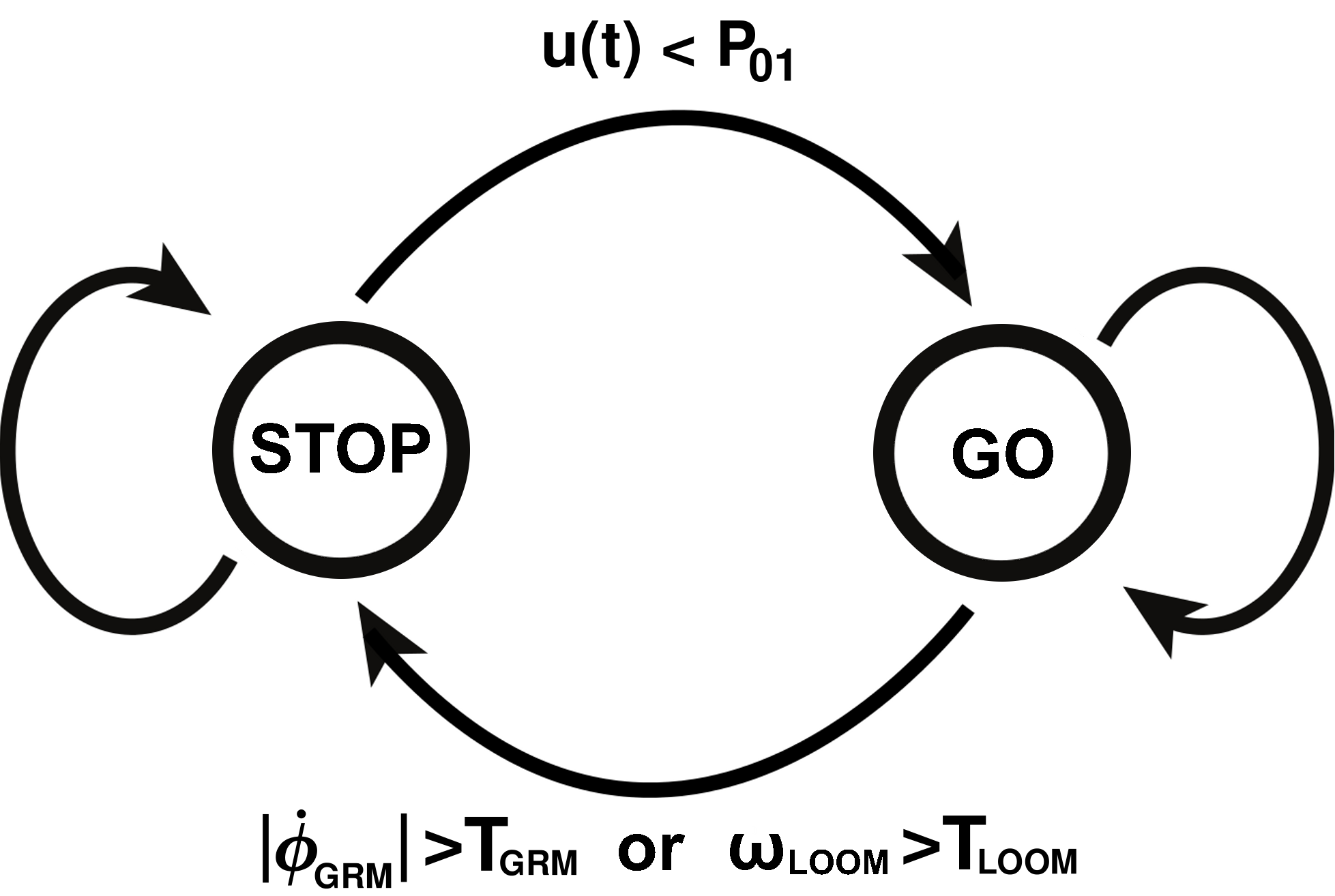}
\caption{Simulated agents, a control diagram. Each fly-like agent keeps following a straight trajectory (see text for details) until either regressive motion or looming expansion on its retina exceeds a fixed threshold. The agent then stops until a Bernoulli coin (tossed on each time step) shows ``heads''.}
\label{fig:fly_control}
\end{figure}

 In order to simplify the notation we omit index $i$ unless necessary, and write out the equations for a one-eyed agent; since both looming and GRM are monocular cues in our implementation, the extension to the two-eyed agent is trivial.

 \begin{description}

 \item[Control:] Let $z \in \{0, 1\}$ be the variable denoting whether
   an agent is stationary ($z=0$) or in motion ($z=1$). Let $u(t)$ be an
   i.i.d. random process with uniform probability density on $(0,
   1)$. Let $\|\dot{\phi}_{GRM}(t)\|$ be the largest observed magnitude of GRM at time t. Let $\omega_{LOOM}(t) := \min{(\|\dot{\phi}_L(t)\|,\|\dot{\phi}_R(t)\|)}$ denote the looming strength observed at time t.

   \begin{align*}
     z(0) &= \begin{array}{l l} 1 & \quad \text{ initially the agent is set in motion } \end{array}\\
     z(t+1) & = \left\{
       \begin{array}{l l} 0 & \quad\text{ if }  z(t)=1 \text{ and}\\
                            & \quad\phantom{W}(\|\dot{\phi}_{GRM}(t)\| >  \text{\Tgrm or}\\
                            & \quad\phantom{W} \omega_{LOOM}(t) > \text{\Tloom})\\
                          1 & \quad \text{ if } z(t)=0\text{ and}\\ 
                            & \quad\phantom{W}\|\dot{\phi}_{GRM}(t)\| < \text{\Tgrm and}\\
                            & \quad\phantom{W} \omega_{LOOM}(t) < \text{\Tloom and }\\
                            & \quad\phantom{W} u(t) < P_{01} \\ 
                       z(t) & \quad\phantom{W} \text{ otherwise}
       \end{array} \right.
   \end{align*}

 \item[Trajectory:] Call $\vec{v}$ the direction vector, i.e. $\|
   \vec{v} \| =1$. With a slight abuse of notation use $\vec{v}$ also
   for the angle of $\vec{v}$, i.e. write $\vec{v} \sim G( \cdot \,; \,
   \mu, \sigma)$ to indicate that the angle of $\vec{v}$ is drawn from
   a given Gaussian density~mod~$2 \pi$.
   \begin{align*}
     x(0) &\sim  U(A)\footnote{}\\ 
     v(0) &\sim U(v_{\min}, v_{\max})\\ 
     \vec{v}(0) & \sim U(0, 2 \pi)\\ 
     x(t+ 1) &= x(t) + v(t)\vec{v}(t) \Delta t\\ 
     v(t+ 1) &= z(t) v\\ 
     \vec{v}(t+1) &\sim \left\{
       \begin{array}{l} 
         G(\cdot \,; \, \vec{v}(t), \sigma_{\vec{v}}(t))\\
           \quad\text{if $z(t)-z(t+1) = 1$}\\ 
           \vec{v}(t) \text{otherwise}
       \end{array}
     \right.
   \end{align*}
   \footnotetext{$A$ is the surface of the walking arena. We re-draw the initial positions so that the flies don't overlap at t=0.}

 \item[Direction change standard deviation:]
   \begin{eqnarray*}
     \sigma_{\vec{v}}(0) & = & 0\\ \sigma_{\vec{v}}(t+1) & = &
     \lambda_\sigma \sigma_{\vec{v}}(t) + \delta_\sigma (1 - z(t))z(t-1) \\
   \end{eqnarray*}
 \end{description}

\section*{DISCUSSION}
Our results show that GRM constitutes a good cue for collision avoidance. Extending the analysis of \citet{Zabala2012}, we showed both mathematically and in simulations that increasing the CVA to a non-zero value improves collision avoidance. In particular, it allows the detection of both stationary and moving objects on a collision course. In this respect, GRM can be viewed as a computationally efficient way to connect looming detection with regressive motion detection. We introduced safety and mobility as collision avoidance performance metrics and used looming as a reference point to show that GRM is a better cue for collision avoidance among conspecifics.

Whereas perception of regressive motion appears to influence the behavior of the fruit fly, the neural circuitry participating in this perception-action loop remains unknown. \citet{Fotowat2009} and \citet{Vries2012} showed that looming-sensitive neurons in \textit{Drosophila} participate in a neural pathway that mediates escape behavior. However, escape is not the same as collision avoidance. We showed that looming might not be a practical collision-avoiding solution for groups of interacting animals as it can overly impede the mobility of a group. We argue that further research into the neural circuitry of GRM-based action in animals is an important future direction.

One reason why regerssive motion has remained a relatively obscure phenomenon might be the often overlooked difference between static and dynamic environments when testing collision avoidance algorithms. For example, \citet{Blanchard2000} constructed a robot guided by responses mimicking those of the locust looming-detection neurons. However, the robot's collision avoidance was only tested in an environment consisting of stationary obstacles. We have shown that in a more interactive environment, GRM has significant advantages over looming.

From the point of view of control and robotics, we studied what \citet{Berg2011} call \textit{reciprocal collision avoidance}: avoiding crashes within a population of simple non-communicating robots all of which implement the same movement protocol. \citeauthor{Berg2011} derive their elegant algorithm from first principles. Ours is biologically inspired and relies on a very simple visual cue that can be computed from the optical flow using Reichardt-detector-like circuits~\citep{Reichardt1961}. While our simulations used agents that can immediately measure the position and velocity of any point on their retina, optic-flow-based motion detection has been successfully implemented on minimalistic hardware~\citep{Barrows2002, Beyeler2009}. We see GRM's largest potential in control of small swarming or flocking vehicles~\citep{Kushleyev2013, Viragh2014}. More sophisticated algorithms for robot navigation are often based on Probabilistic Roadmaps~\citep{Kavraki1996, Boor1999, Karaman2011} or Rapidly Exploring Random Trees~\citep{LaValle1998,Petti2005, Kuwata2009, Karaman2011}. Such approaches allow a robot to avoid obstacles with lower failure rates than GRM, but require large computational power. Importance of low-complexity collision avoidance grows as fields such as drone flight control are rapidly developing~\citep{Pines2006, Lentink2014}. Alternatively, GRM could serve as a basis for a fallback collision prevention system for larger vehicles operating in highly dynamic environments.

\section*{Acknowledgements}
This work was supported by National Science Foundation grants 0914783 and 1216045, NASA Stennis grant NAS7.03001 and ONR MURI grant N00014-10-1-0933.

\bibliographystyle{plainnat}
\bibliography{GRM}
\clearpage

\section*{MATHEMATICAL APPENDIX}

\textbf{Proposition 2}
{\it
  Let the relative position and velocity of the
  observed object be $\mathbf{x} \text{ and } \mathbf{v}$
  respectively. Then $\dot{\phi} = \frac{1}{\|\mathbf{x}\|^2}\langle \mathbf{v}^\perp,
  \mathbf{x} \rangle$. In particular, angular velocity scales as one over
  distance squared.
}
\begin{proof} 
  We wish to derive the angular velocity of a point in relative
  motion projecting onto an observer. Place the center of
  projection at the origin, and a particle moving with constant velocity
  $\mathbf{v}=(u,v)$ at position $\mathbf{x_0}=(x_0,y_0)$ at time $0$, as shown
  in Fig. \ref{fig:angvel_derive}. 

  \begin{figure}[ht!]
    \centering
    \includegraphics[width=2in]{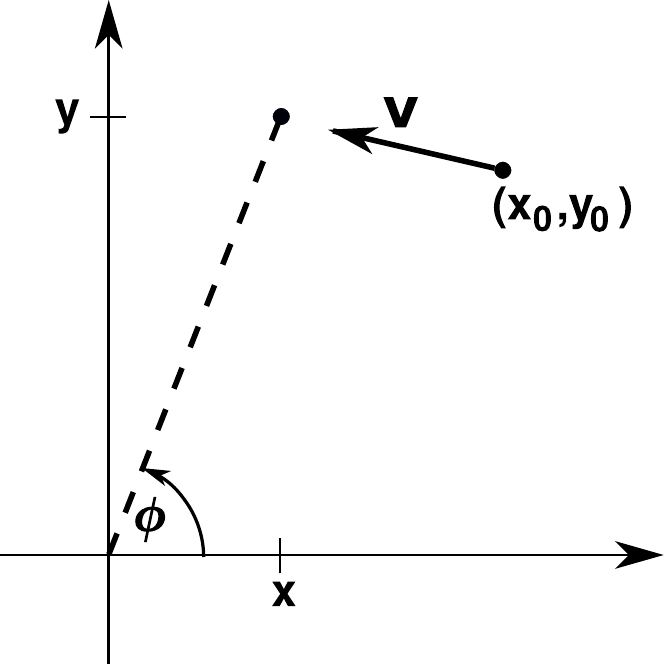}
    \caption{Calculating angular velocity. A particle moving with
      velocity $\mathbf{v}$ projects angular velocity $\dot{\phi}$ on the origin.}
    \label{fig:angvel_derive}
  \end{figure}

  The position at time $t$ equals
  $\mathbf{x}=\mathbf{x_0}+t\mathbf{v}$, and the azimuth of the particle $\phi$
  is such that
  \[
  \tan{\phi} = \frac{\mathbf{y}}{\mathbf{x}} = \frac{y_0+vt}{x_0+ut}.
  \]
  Taking the time derivative on both sides gives
  \[
  \frac{1}{\cos^2{\phi}}\dot{\phi} = \frac{v(x_0+tu)-u(y_0+tv)}{(x_0+ut)^2},
  \]
  and thus
  \begin{align}
    \dot{\phi} &= \frac{cos^2{\phi}(v(x_0+tu)-u(y_0+tv))}{(x_0+tu)^2}\\
    &= \frac{(x_0+tu)^2(v(x_0+tu)-u(y_0+tv))}{D^2(x_0+tu)^2},\label{eq:dotphi}
  \end{align}
  where $D$ is the distance of the particle from the origin. Equation
  \ref{eq:dotphi} follows from the relation
  $\cos{\phi}=\frac{x_0+tu}{D}$ (see Fig.
  \ref{fig:angvel_derive}). Simplifying the RHS yields
  \begin{align}
    \dot{\phi} &= \frac{v(x_0+tu)-u(y_0+tv)}{D^2}\\
    &= \frac{1}{\|\mathbf{x}\|^2}\langle \mathbf{v}^\perp, \mathbf{x} \rangle.\label{eq:angvel_derivation}
  \end{align}
  Since the position $\mathbf{x}$ moves on a line perpendicular to
  $\mathbf{v}^\perp$ we expect the dot product to be constant, and indeed it
  equals $vx_0-uy_0$. Hence angular velocity decays as one over distance
  squared.
\end{proof}

\section*{MOVIE CAPTIONS APPENDIX}
Available online are supplementary Movies S1-S4. The movies show full trials of our simulations for chosen parameter settings. In each movie, ten fly-like vehicles are visible, colored arbitrarily to make tracking the vehicles easy. Whenever two vehicles collide, their body size is temporarily increased. Whenever a vehicle stops, it is surrounded by a colored circle. A red circle means the stop is a False Positive. A green circle indicates a True Positive. In addition, a line segment is drawn from the stopping vehicle to the one (or more) causes of its stop.

\begin{compactitem}
\item[Movie S1: CVA=10$^\circ$, T\textsubscript{GRM}=6rad/s, T\textsubscript{LOOM}=32rad/s.] In this movie, GRM is the stopping mechanism, and the CVA is small. An interesting situation arises at about 00:10, where three vehicles (bright green, blue, and yellow) meet. Green stops due to blue's motion, but unnecessarily. Blue avoids a collision with yellow. Yellow in turn crashes into green. That is because green is already stationary, so the GRM magnitude it evokes on yellow's retina is relatively small, and the small CVA prevents yellow from picking up any strong signals from stationary obstacles.

\item[Movie S2: CVA=70$^\circ$, T\textsubscript{GRM}=6rad/s, T\textsubscript{LOOM}=32rad/s.] This time, the CVA is large. This makes it easy for the flies to detect stationary obstacles on time. There are few collisions, but many unnecessary stops. The collision at 00:22 (blue and bright-green) is a good example of the type of collision that is hard to avoid using GRM detection. The vehicles' relative motion is insignificant, making the evoked GRM signal small.

\item[Movie S3: CVA=10$^\circ$, T\textsubscript{GRM}=32rad/s, T\textsubscript{LOOM}=6rad/s.] In this case the CVA is small and looming is the significant stopping mechanism. Collisions with stationary obstacles are hard to detect, mainly because the CVA is rather small (for example, three of them happen roughly at the same time at 00:04). 

\item[Movie S4: CVA=70$^\circ$, T\textsubscript{GRM}=32rad/s, T\textsubscript{LOOM}=6rad/s.] Looming with large CVA. Encounters such as the light-blue fly stopping at 00:15 emphasize that simple looming mechanism (such as the one used in our simulation) don't know about figure-ground segmentation. The two flies that caused the stop are perceived as one expanding entity on light-blue's eye.

\end{compactitem}

\end{document}